\documentclass[sigconf]{acmart}

\AtBeginDocument{%
  \providecommand\BibTeX{{%
    \normalfont B\kern-0.5em{\scshape i\kern-0.25em b}\kern-0.8em\TeX}}}

\copyrightyear{2023}
\acmYear{2023}
\setcopyright{rightsretained}
\acmConference[WSDM '23]{Proceedings of the Sixteenth ACM International Conference on Web Search and Data Mining}{February 27-March 3, 2023}{Singapore, Singapore}
\acmBooktitle{Proceedings of the Sixteenth ACM International Conference on Web Search and Data Mining (WSDM '23), February 27-March 3, 2023, Singapore, Singapore}\acmDOI{10.1145/3539597.3570376}
\acmISBN{978-1-4503-9407-9/23/02}

\usepackage{tikz}
\usetikzlibrary{arrows.meta}

\DeclareMathOperator{\gain}{\textbf{gain}}

\DeclareMathOperator{\GED}{GED}
\usepackage{mathtools}

\DeclareMathOperator*{\argmin}{arg\,min}
\DeclareMathOperator*{\argmax}{arg\,max}

\usepackage[normalem]{ulem}
\usepackage{multirow}
\usepackage{float}
\usepackage{enumitem}
\usepackage{multicol}
\usepackage{chemfig}
\usepackage{caption}
\usepackage[font=small]{subfig}
\usepackage{xspace}

\newcommand{\name}{\textsc{GCFExplainer}\xspace}
\DeclareMathOperator{\Cost}{\textbf{cost}}
\DeclareMathOperator{\Cover}{\textbf{coverage}}
\DeclareMathOperator{\Size}{\textbf{size}}
\DeclareMathOperator*{\agg}{agg}

\usepackage{algorithm}
\usepackage[noend]{algorithmic}

\newtheorem{thm}{\textbf{Theorem}}

\newtheorem{definition}{\textbf{Definition}}

\newtheorem{problem}{\textbf{Problem}}

\setlist{nolistsep,leftmargin=*}
\makeatletter
\tikzset{
    dot diameter/.store in=\dot@diameter,
    dot diameter=1pt,
    dot spacing/.store in=\dot@spacing,
    dot spacing=5.5pt,
    dots/.style={
        line width=\dot@diameter,
        line cap=round,
        dash pattern=on 0pt off \dot@spacing
    }
}
\makeatother

\begin{document}

\title{Global Counterfactual Explainer for Graph Neural Networks}
\author{Zexi Huang}
\authornote{Both authors contributed equally to this research.}
\affiliation{%
 \institution{University of California}
 \city{Santa Barbara}
 \state{CA}
 \country{USA}
}
\email{zexi_huang@cs.ucsb.edu}

\author{Mert Kosan}
\authornotemark[1]
\affiliation{%
 \institution{University of California}
 \city{Santa Barbara}
 \state{CA}
 \country{USA}}
\email{mertkosan@cs.ucsb.edu}

\author{Sourav Medya}
\affiliation{%
 \institution{University of Illinois}
 \city{Chicago}
 \state{IL}
 \country{USA}}
\email{medya@uic.edu}

\author{Sayan Ranu}
\affiliation{%
 \institution{Indian Institute of Technology}
 \city{Delhi}
 \country{India}}
\email{sayanranu@cse.iitd.ac.in}

\author{Ambuj Singh}
\affiliation{%
 \institution{University of California}
 \city{Santa Barbara}
 \state{CA}
 \country{USA}}
\email{ambuj@cs.ucsb.edu}

\begin{abstract}
Graph neural networks (GNNs) find applications in various domains 
such as computational biology, natural language processing, and computer security.
Owing to their popularity, there is an increasing need to explain GNN predictions since GNNs are black-box machine learning models. One way to address %
this is \textit{counterfactual} reasoning where the objective is to change the GNN prediction by minimal changes in the input graph. %
Existing methods for counterfactual explanation of GNNs are limited to instance-specific \emph{local} reasoning. 
This approach has two major limitations of not being able to offer global recourse policies and overloading human cognitive ability with too much information. In this work, we study the \textit{global} explainability of GNNs through global counterfactual reasoning. Specifically, we want to find a \textit{small} set of representative counterfactual graphs that explains \textit{all} input graphs. %
Towards this goal, we propose \name, a novel algorithm powered by \textit{vertex-reinforced random walks} on an \textit{edit map} of graphs with a \emph{greedy summary}.
Extensive experiments on real graph datasets show that the global explanation from \name provides important high-level insights of the model behavior and achieves a \textbf{46.9\%} gain in recourse coverage and a \textbf{9.5\%} reduction in recourse cost compared to the state-of-the-art local counterfactual explainers. 
\looseness=-1
\end{abstract}



\settopmatter{printfolios=true}
\maketitle
\section{Introduction}
Graph Neural Networks (GNNs) \cite{kipf2016semi, hamilton2017inductive, velivckovic2018graph, wang2021gnnadvisor, wang2022qgtc,graphreach} are being used in many domains such as drug discovery~\cite{jiang2020drug}, chip design~\cite{chip}, %
combinatorial optimization~\cite{gcomb}, physical simulations~\cite{lgnn,lgnn_benchmarking} and event prediction~\cite{kosan2021event,medya_earnings_22,tigger}. %
Taking the graph(s) as input, GNNs are trained to perform various downstream tasks %
that form the core of many real-world applications. For example, graph classification has been applied to predict whether a drug would exhibit the desired chemical activity~\cite{jiang2020drug}. Similarly, node prediction is used to predict the functionality of proteins in protein-protein interaction networks~\cite{borgwardt2005protein} and categorize users into roles on social networks~\cite{social}. \looseness=-1

Despite the impressive success of GNNs on predictive tasks, GNNs are \textit{black-box} machine learning models. It is non-trivial to explain or reason why a particular prediction is made by a GNN. Explainability of a prediction model is important to understand its shortcomings and identify areas for improvement. %
In addition, the ability to explain a model is critical towards making it trustworthy. Owing to this limitation of GNNs, there has been 
significant efforts in recent times towards explanation approaches. 
\looseness=-1

Existing work on explaining GNN predictions can be categorized mainly in two directions: 1) factual reasoning \cite{ying2019gnnexplainer,luo2020parameterized,vu2020pgm,xgnn_kdd20}, and 2) counterfactual reasoning \cite{lucic2021cf,bajaj2021robust,abrate2021counterfactual,tan2022learning}. Generally speaking, the methods in the first category aim to find an important subgraph that correlates most with the underlying GNN prediction. In contrast, the methods with counterfactual reasoning attempt
to identify the smallest amount of perturbation on the input graph that changes the GNN’s prediction, for example, removal/addition of edges or nodes. 
\looseness=-1

\begin{figure}[b]
\vspace{-0.28in}
    \centering
         \subfloat[Formaldehyde]{\label{subfig::formaldehyde}\includegraphics[height=0.1\textwidth]{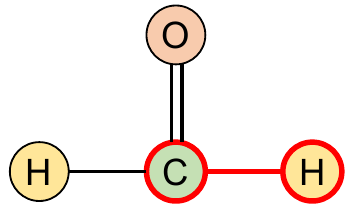}}
    \quad\quad   
    \subfloat[Formic acid]{\label{subfig::formicacid}\includegraphics[height=0.1\textwidth]{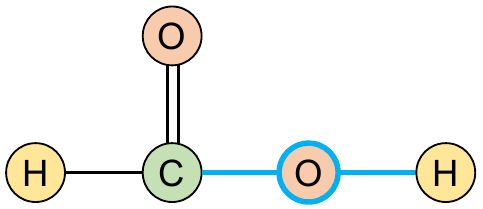}}
    \vspace{-0.12in}
    \caption{%
    Formaldehyde (a) is classified by a GNN to be an undesired mutagenic molecule with its important subgraph found by factual reasoning highlighted in red. Formic acid (b) is its non-mutagenic counterfactual example obtained by removing one edge and adding one node and two edges. %
    \looseness=-1}
    \label{fig::example_counterfactual}
\end{figure}

Compared to factual reasoning, counterfactual explainers have the additional advantage of providing the means for recourse \cite{voigt2017eu}. For example, in the applications of drug discovery \cite{jiang2020drug, xiong2021graph}, mutagenicity is an adverse property of a molecule that hampers its
potential to become a marketable drug \cite{kazius2005derivation}. In \autoref{fig::example_counterfactual}, formaldehyde is classified by a GNN to be mutagenic. Factual explainers can attribute the subgraph containing the carbon-hydrogen bond to the cause of mutagenicity, while counterfactual explainers provide an effective way (i.e., a recourse) to turn formaldehyde into formic acid, which is non-mutagenic, by replacing a hydrogen atom with a hydroxyl.  
\looseness=-1

In this work, we focus on counterfactual explanations. Our work is based on the observation that existing counterfactual explainers \cite{ying2019gnnexplainer,luo2020parameterized,vu2020pgm,xgnn_kdd20} for graphs take a \textit{local} perspective, generating counterfactual examples for individual input graphs. However, this approach has two key limitations: 
\begin{itemize}
\item {\bf Lack of global insights: } It is desirable to offer insights that generalize across a multitude of data graphs. For example, instead of providing formic acid as a counterfactual example to formaldehyde, we can summarize global recourse rules such as ``\emph{Given any molecule with a carbonyl group (carbon-oxygen double bond), it needs a hydroxy to be non-mutagenic}''. This focus on global counterfactual explanation promises to provide higher-level insights that are complementary to those obtained from local counterfactual explanations. 

\item {\bf Information overload:} The primary motivation behind counterfactual analysis is to provide human-intelligible explanations. With this objective, consider real-world graph datasets that routinely contain thousands to millions of graphs. Owing to instance-specific counterfactual explanations, the number of counterfactual graphs grows linearly with the graph dataset size. Consequently, the sheer volume of counterfactual graphs overloads human cognitive ability to process this information. Hence, the initial motivation of providing human-intelligible insights is lost if one does not obtain a holistic view of the counterfactual graphs.
\looseness=-1
\end{itemize}

\vspace{\baselineskip}
\noindent
\textbf{Contributions:} In this paper, we study the problem of model-agnostic, \emph{global} counterfactual explanations of GNNs for graph classification. More specifically, given a graph dataset, our goal is to counterfactually explain the largest number of input graphs with a small number of counterfactuals. As we will demonstrate later in our experiments, this formulation naturally forces us to remove redundancy from instance-specific counterfactual explanations and hence has higher information density. Algorithmically, the proposed problem introduces new challenges. We theoretically establish that the proposed problem is NP-hard. Furthermore, the space of all possible counterfactual graphs itself is exponential. Our work overcomes these challenges and makes the following contributions: \looseness=-1

\begin{itemize}
    \item \textbf{Novel formulation:} We formulate the novel problem of global counterfactual reasoning/explanation of GNNs for graph classification. In contrast to existing works on counterfactual reasoning that only generate instance-specific examples, we provide an explanation on the global behavior of the model.
    \item {\bf Algorithm design:} While the problem is NP-hard, we propose \name, which organizes the exponential search space as an \textit{edit map}. We then perform \textit{vertex-reinforced random walks} on it to generate diverse, representative counterfactual candidates, which are \textit{greedily summarized} as the global explanation. 
    \item {\bf Experiments:} We conduct extensive experiments on real-world datasets to validate the effectiveness of the proposed method. Results show that \name not only provides important high-level insights on the model behavior but also outperforms state-of-the-art baselines related to counterfactual reasoning in various recourse quality metrics. 
\end{itemize}

\section{Global Counterfactual Explanations }
This section introduces the global counterfactual explanation (GCE) problem for graph classification. We start with the background on local counterfactual reasoning. %
Then, we propose a representation of the global recourse rule that provides a high-level counterfactual understanding of the classifier behavior. Finally, we introduce quality measures for recourse rules and formally define the GCE problem. \looseness=-1

\subsection{Local Counterfactual}
Consider a graph $G=(V,E)$, where $V$ and $E$ are the sets of (labelled) nodes and edges respectively. A (binary) graph classifier (e.g., a GNN) $\phi$ classifies $G$ into either the undesired class ($\phi(G) = 0$) or the desired one ($\phi(G)=1$). An explanation of $\phi$ seeks to answer how these predictions are made. Those based on factual reasoning %
analyze what properties $G$ possesses to be classified in the current class while those based on counterfactual reasoning 
find what properties $G$ needs to be assigned to the opposite class. 
 
Existing counterfactual explanation methods take a local perspective. Specifically, for each input graph $G$, they find a \textbf{counterfactual} (graph) $C$ that is somewhat similar to $G$ but is assigned to a different class. Without loss of generality, let $G$ belong to the undesired class, i.e., $\phi(G)=0$, then the counterfactual $C$ satisfies $\phi(C)=1$. The similarity between $C$ and $G$ is quantified by a predefined distance metric $d$, for example, the number of added/removed edges \cite{lucic2021cf, bajaj2021robust}. \looseness=-1

In our work, we consider the graph edit distance (GED) \cite{sanfeliu1983distance}, a more general distance measure, as the distance function to account for other types of changes. Specifically, $\GED(G_1, G_2)$ counts the minimum number of ``edits'' to convert $G_1$ to $G_2$. An ``edit'' can be the addition or removal of edges and nodes, or change of node labels (see \autoref{fig::edit_map}). Moreover, to account for graphs of different sizes, we normalize the GED by the sizes of graphs: $\widehat{\GED}(G_1, G_2) = \GED(G_1, G_2) / (|V_1| + |V_2| + |E_1| + |E_2|)$. Nonetheless, our method can be applied with other graph distance metrics, such as those based on graph kernels (e.g., RW \cite{borgwardt2006fast}, NSPDG \cite{costa2010fast}, WL \cite{shervashidze2011weisfeiler}). 
\begin{figure}[htbp]
    \centering
    \includegraphics[width=0.8\columnwidth]{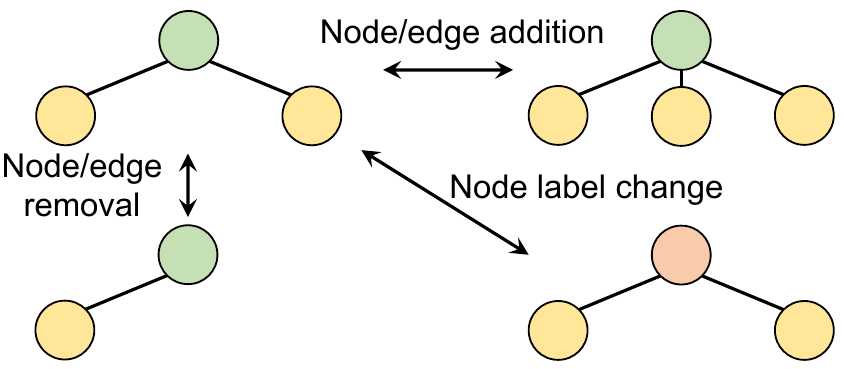}
    \vspace{-0.15in}
    \caption{Edits between graphs. }
    \label{fig::edit_map}
\vspace{-0.05in}
\end{figure}

The distance function measures the quality of the counterfactual found by the explanation model. Ideally, the counterfactual $C$ should be very close to the input graph $G$ while belonging to a different class. Formally, we define the counterfactuals that are within a certain distance $\theta$ from the input graph as \textbf{close counterfactuals}.  %

\begin{definition}[Close Counterfactual]
 Given the GNN classifier $\phi$, distance parameter $\theta$, and an input graph $G$ with undesired outcome, i.e., $\phi(G)=0$; a counterfactual graph, $C$, is a close counterfactual of $G$ when $\phi(C)=1$ and $d(G, C) \leq \theta$.
 \end{definition}

While the (close) counterfactual $C$ found by existing methods explains the classifier behavior for the corresponding input graph $G$, it is hard to generalize to understand the global pattern. Next, we introduce the global recourse rule that provides a high-level summary of the classifier behavior across different input graphs. \looseness=-1

\subsection{Global Recourse Representation}
The global counterfactual explanation requires a global recourse rule $r$. Specifically, for any (undesired) input graph $G$ with $\phi(G)=0$, $r$ provides a (close) counterfactual (i.e., a recourse) for $G$: $\phi(r(G)) = 1$. While both a recourse rule and a local counterfactual explainer find a counterfactual given an input graph, their goals are different. The goal of the local counterfactual explainer is to find the best (closest) counterfactual possible for each input graph, and therefore, $r$ can be very complicated, e.g., in the form of an optimization algorithm \cite{bajaj2021robust, tan2022learning}. On the other hand, a recourse rule aims to provide an explanation of the classifier's global behavior, which requires a simpler form that is understandable for domain experts without prior knowledge of deep learning on graphs. %

Existing global recourse rules for classifiers with feature vectors as input take the form of short decision trees \cite{rawal2020beyond}. However, this is hard to be generalized to graph data with rich structure information. Instead, we propose the representation of a global recourse rule for a graph classifier to be a collection of counterfactual graphs $\mathbb{C}$ in the desired class that are \textit{diverse} and \textit{representative} enough to capture its global behavior. This representation does not require any additional knowledge for domain experts to understand and draw insights from, similar to the local counterfactual examples. It is also easy to find the local counterfactual for a given input graph $G$ based on $\mathbb{C}$ by nominating the closest graph in $\mathbb{C}$: $r(G) = \argmin_{C \in \mathbb{C}} d(G, C)$.

\subsection{Quantifying Recourse Quality}
Given a graph classifier $\phi$ and a set of $n$ input graphs $\mathbb{G}$ in the undesired class, we want to compare the quality of different recourse representations $\mathbb{C}$. Similar to the quality metrics introduced for vector data \cite{rawal2020beyond}, we aim to account for the following factors: 
\begin{enumerate}
    \item \textbf{Coverage}: Like local counterfactual explainers, we want to ensure that counterfactuals found for individual input graphs are of high quality. Specifically, we introduce recourse coverage---the proportion of input graphs that have close counterfactuals from $\mathbb{C}$ under a given distance threshold $\theta$: 
    $$\Cover(\mathbb{C}) = |\{G \in \mathbb{G} \mid \min_{C \in \mathbb{C}} \left\{d(G, C)\right\} \leq \theta\}| / |\mathbb{G}|$$
    \item \textbf{Cost}: Another quality metric based on local counterfactual quality is the recourse cost (i.e., the distance between the input graph and its counterfactual) across the input graphs:
	$$\Cost(\mathbb{C}) = \agg_{G \in \mathbb{G}} \{\min_{C \in \mathbb{C}} \{d(G, C)\}\}$$
	where $\agg$ is an aggregation function, e.g., mean or median. 
    \item \textbf{Interpretability}: Finally, the recourse rule should be easy (small) enough for human cognition. We quantify the interpretability as the size of recourse representation:
    $$\Size(\mathbb{C}) = |\mathbb{C}|$$
\end{enumerate}
\subsection{Problem Formulation and Characterization}
An ideal recourse representation should maximize the coverage while minimizing the cost and the size. Formally, we define the global counterfactual explanation problem as follows: 

\begin{problem}[Global Counterfactual Explanation for Graph Classification (GCE)] 
Given a GNN graph classifier $\phi$ that classifies $n$ input graphs $\mathbb{G}$ to the undesired class 0 and a budget $k \ll n$, our goal is to find the best recourse representation $\mathbb{C}$ that maximizes the recourse coverage with size $k$: 
$$\max_{\mathbb{C}}\Cover(\mathbb{C}) \text{ s.t. }\Size(\mathbb{C}) = k$$
\end{problem}

We note that in our problem formulation only coverage and size are explicitly accounted for, whereas cost is absent. We make this design choice since cost and coverage are intrinsically opposing forces. Specifically, if we are willing to allow a high cost, coverage increases since we allow for higher individual distances between an input graph and its counterfactual. Therefore, we take the approach of binding the cost to the distance threshold $\theta$ in the coverage definition. Nonetheless, an explicit analysis of all these metrics including cost is performed to quantify recourse quality during our empirical evaluation in Section~\ref{sec:experiments}. Below we discuss the hardness of GCE. \looseness=-1
\begin{thm}
[NP-hardness]
\label{thm:nphard}
The GCE problem is NP-hard.
\end{thm}
\begin{proof}

To establish NP-hardness of the proposed problem we reduce it from the classical \textit{Maximum Coverage} problem.
\begin{definition}
[Maximum Coverage]
Given a budget $k$ and a collection of subsets $\mathcal{S}=\{S_{1},\cdots,S_{m}\}$ from a universe of items $U=\{u_{1},\cdots,u_{n}\}$, find a subset $\mathcal{S}'\subseteq \mathcal{S}$ of sets such that $|\mathcal{S}'|\leq k$ and the number of covered elements $|\bigcup_{\forall S_i\in S'} S_i|$ is maximized.
\end{definition}

We show that given any instance of a maximum coverage problem $\langle \mathcal{S}, U\rangle$, it can be mapped to a GCE problem. %
For $u_i$, we construct a star graph with a center node with an empty label and $n$ leaf nodes with $n-1$ empty labels and one label $u_i$. For $S_i$, we construct a similar star graph with a center node with a special label $\gamma$ and $n$ leaf nodes with $|S_i|$ labeled with the elements in $S_i$ and $n-|S_i|$ with empty labels. The classifier $\phi$ classifies a graph as a desired one if and only if it is a star graph with a $\gamma$-labeled central node and $n$ leaf nodes with a set of labels among $\mathcal{S}=\{S_{1},\cdots,S_{m}\}$. The allowed edit operations are either adding or deleting a set of labels (as a single edit), but not both together. So, each $S_i$ corresponds to a counterfactual candidate $C_i$ and $d(G_j, C_i) \leq \theta = 1$ if and only if $u_j\in S_i$. With this construction, it is easy to see that an optimal solution for this instance of GCE is the optimal solution for the corresponding instance of the maximum coverage problem. 
\end{proof} 

Owing to NP-hardness, it is not feasible to identify the optimal solution for the GCE problem in polynomial time unless NP $=$ P.  In the next section, we will introduce \name, an effective and efficient heuristic that solves the GCE problem.

\section{Proposed Method: \name}
In this section, we propose \name, the first global counterfactual explainer for graph classification. The GCE problem requires us to find a collection of $k$ counterfactual graphs that maximize the coverage of the input graphs. Intuitively, we want each individual counterfactual graph to be a close counterfactual to (i.e., ``cover'') as many input graphs as possible. Additionally, different counterfactual graphs should cover different sets of input graphs to maximize the overall coverage. These intuitions motivate the design of our algorithm \name, which has three major components: 
\begin{enumerate}
    \item \textbf{Structuring the search space:} The search space of counterfactual graphs consists of \textit{all} graphs that are in the same domain as the input graphs and within a distance of $\theta$. In other words, any graph within a distance of $\theta$ from an input graph may be a potential counterfactual candidate and therefore needs to be analyzed. The number of potential graphs within $\theta$ increases exponentially with $\theta$ since the space of graph edits is combinatorial~\cite{ged1,ranjan2021neural}. %
    \name uses an \textit{edit map} to organize these graphs as a meta-graph $\mathcal{G}$, where individual nodes are graphs that are created via a different number of edits from the input graphs and each edge represents a single edit. 
    
    \item \textbf{Vertex-reinforced random walk:} To search for good counterfactual candidates, \name leverages vertex-reinforced random walks (VRRW) \cite{pemantle1992vertex} on the edit map $\mathcal{G}$. VRRW has the nice property of converging to a set of nodes that are both important (i.e., cover many input graphs) and diverse (i.e., non-overlapping coverage), which will form a small set of counterfactual candidates for further processing. 
    
    \item \textbf{Iterative computation of the summary:} After obtaining good counterfactual candidates from VRRW, \name creates the final set of the counterfactual graphs (i.e., the summary) as the recourse representation by iteratively adding the best candidate based on the maximal gain of the coverage given the already added candidates. 
    
\end{enumerate}

\subsection{Structuring the Search Space}
The search space for counterfactual graphs in \name is organized via an edit map $\mathcal{G}$. The edit map is %
a meta-graph whose nodes are graphs in the same domain as the input graphs and edges connect graphs that differ by a single graph edit. As an example, each graph in \autoref{fig::edit_map} represents a node in the edit map, and the arrows denote edges between graphs (nodes) that are one edit away. In the edit map, we only include connected graphs since real graphs of interest are often connected (e.g., molecules, proteins, etc.). 

While all potential counterfactual candidates are included as its nodes, the edit map has an exponential size %
and it is computationally prohibitive to fully explore it. However, a key observation is that a counterfactual candidate can only be a few hops away from some input graph. Otherwise, the graph distance between the counterfactual and the input graph would be too large for the counterfactual to cover it. This observation motivates our exploration of the edit map to be focused on the union of close neighborhoods of the input graphs (see Section \ref{subsubsec::teleportation}). Additionally, while we cannot compute the entire edit map, it is easy to chart the close neighborhoods by iteratively performing all possible edits from the input graphs. Next, we introduce the vertex-reinforced random walk to efficiently explore the edit map to find counterfactual candidates.

\subsection{Vertex-Reinforced Random Walk}

Vertex-reinforced random walk (VRRW) \cite{pemantle1992vertex} is a time-variant random walk. Different from other more widely applied random walk processes such as the simple random walk and the PageRank \cite{perozzi2014deepwalk, klicpera2018predict, huang2021broader, huang2022pole}, the transition probability $p(u, v)$ of VRRW from node $u$ to node $v$ depends not only on the edge weight $w(u, v)$ but also the \textit{number of previous visits} in the walk to the target node $v$, which we denote using $N(v)$. Specifically, 
\begin{equation}
   p(u, v) \propto w(u, v)N(v)
   \label{eq::transition}
\end{equation}

\name applies VRRW on the edit map and produces $n$ most frequently visited nodes in the walk as the set of counterfactual candidates $\mathbb{S}$. Next, we formalize VRRW in our setting and explain how it surfaces good counterfactual candidates for GCE. 

\subsubsection{Vertex-reinforcement}
Our main motivation for using VRRW to explore the edit map instead of other random walk processes is that VRRW converges to a diverse and representative set of nodes \cite{mei2010divrank, natarajan2016scalable} in different regions of the edit map. In this way, the frequently visited nodes in instances of VRRW have the potential to be good counterfactual candidates as they would cover a diverse set of input graphs in the edit map. The reason behind the diversity of the highly visited nodes is the previous visit count $N(v)$ in the transition probability. Specifically, nodes with larger visit counts tend to be visited more often later (``richer gets richer''), and thereby dominating all other nodes in their neighborhood. This leads to a bunch of highly visited nodes to ``represent'' each region of the edit map. We refer the readers to \cite{mei2010divrank} for details on the mathematical basis and the theoretical correctness of this property. Moreover, as our goal is to find counterfactual candidates, we only reinforce (i.e., increase the visit counts of) graphs in the counterfactual class. 

\subsubsection{Importance function}

While the vertex-reinforcement mechanism ensures diversity of the highly visited nodes, we still need to guide the walker to visit graphs that are good counterfactual candidates. We achieve this by assigning large edge weight $w(u, v)$ to good counterfactual candidates via an importance function $I(v)$:
\begin{equation}
    w(u, v)=I(v)
\end{equation}
The importance function $I(v)$ should capture the quality of a graph $v$ as a counterfactual candidate. It has the following components:
\begin{enumerate}
    \item Counterfactual probability $p_\phi(v)$. The graph classifier $\phi$ predicts a probability for $v$ to be in the counterfactual class ($\phi(v)=1$). By using it as part of the importance function, the walker is encouraged to visit regions with rich counterfactual graphs.
    \item Individual coverage $\Cover(\{v\})$. The individual coverage of a graph $v$ computes the proportion of input graphs that are close to $v$. This encourages the walker to visit graphs that cover a large number of input graphs. 
    \item Gain of coverage $\gain(v;\mathbb{S})$. Given a graph $v$ and the current set of counterfactual candidates $\mathbb{S}$ (i.e., the $n$ most frequently visited nodes), we can compute the gain between the current coverage and the coverage after adding $v$ to $\mathbb{S}$:  
    $$\gain(v;\mathbb{S}) = \Cover(\mathbb{S}\cup\{v\})- \Cover(\mathbb{S})$$
    This guides the walker to find graphs that complement the current counterfactual candidates to cover additional input graphs. 
\end{enumerate}
The importance function is a combination of these components: 
\begin{equation}
    I(v) = p_\phi(v)(\alpha\Cover(\{v\}) +(1-\alpha)\gain(v;\mathbb{S}))
    \label{eq::importance}
\end{equation}
where $\alpha$ is a hyperparameter between 0 and 1. With the above importance function, the VRRW in \name converges to a set of diverse nodes that have high counterfactual probability and collectively cover a large number of input graphs. 

\subsubsection{Dynamic teleportation}\label{subsubsec::teleportation}
The last component of VRRW, teleportation, is to help us manage the exponential search space of the edit map. Since our goal is to find close counterfactuals to the input graphs, the walker only needs to explore the nearby regions of the input graphs. Therefore, we start the walk from the input graphs, and also at each step, let the walker teleport back (i.e. transit) to a random input graph with probability $\tau$. 

To decide which input graph to teleport to, we adopt a dynamic probability distribution based on the current counterfactual candidate set $\mathbb{S}$. Specifically, let $g(G)=|\{v \in \mathbb{S} \mid d(v, G) \leq \theta \text{ and } \phi(v)=1\}|$ be the number of close counterfactuals in $\mathbb{S}$ covering an input graph $G$. Then the probability to teleport to $G$ is 
\begin{equation}
    p_{\tau}(G) = \frac{\exp (-g(G))}{\sum_{G' \in \mathbb{G}} \exp (-g(G'))}
    \label{eq::teleportation}
\end{equation}

This dynamic teleportation favors input graphs that are not well covered by the current solution set and encourages the walker to explore nearby counterfactuals to cover them after teleportation. %

\subsection{Iterative Computation of the Summary}
\label{subsec::summary}
We have applied VRRW to generate a good set of $n$ counterfactual candidates $\mathbb{S}$. In the last step of \name, we aim to further refine the candidate set and create the final recourse representation (i.e., the summary) with $k$ counterfactual graphs. This summarization problem is also NP-hard and we propose to build $\mathbb{C}$ in an iterative and greedy manner from $\mathbb{S}$.

Specifically, we start with an empty solution set $\mathbb{C}_0$. Then, for each iteration $t$, we add the graph $v$ to $\mathbb{C}_t$ with the maximal gain of coverage $\gain(v; \mathbb{C}_t)$. This is repeated $k$ times to get the final recourse representation $\mathbb{C}$ with $k$ graphs. It is easy to show that the summarization problem is submodular and therefore, our greedy algorithm provides $(1-1/e)$-approximation. 

Notice that the greedy algorithm can also be applied to the local counterfactuals found by existing methods to generate a GCE solution. Here, we highlight three advantages of \name: 
\begin{enumerate}
    \item Existing local counterfactual explainers \cite{lucic2021cf,bajaj2021robust,abrate2021counterfactual,tan2022learning} are only able to generate counterfactuals based on one type of graph edits---edge removal, while \name incorporates all types of edits to include a richer set of counterfactual candidates. 
    \item The set of counterfactual candidates from \name is generated with the GCE objective in mind, while the local counterfactuals from existing methods are optimized for individual input graphs. Therefore, they may not be good candidates to capture the global behavior of the classifier.
    \item It is easy to incorporate domain constraints (e.g., the valence of chemical bonds) into \name by pruning the neighborhood of the edit map, while existing methods based on optimization require non-trivial efforts to customize. 
\end{enumerate}
We will empirically demonstrate the superiority of \name to this two-stage approach with state-of-the-art local counterfactual explanation methods in our experiments in Section~\ref{subsec::quality}. 
\vspace{\baselineskip}

\noindent\textbf{Pseudocode and complexity:} The pseudocode of \name is presented in \autoref{alg::gcf}. Line 1-16 summarizes the VRRW component of \name. Specifically, Line 3-10 determines the next graph to visit based on VRRW transition probabilities and dynamic teleportation, and Line 11-16 update the visit counts and the set of counterfactual candidates. 
The iterative computation of the counterfactual summary is described in Line 17-21. The overall complexity of \name is $O(Mhn + kn)$, where $M$ is the number of iterations for the VRRW, $h$ is the average node degree in the meta-graph, $n$ is the number of input graphs, and $k$ is the size of the global counterfactual representation. In practice, we store the computed transition probabilities with a space-saving algorithm \cite{metwally2005efficient} to improve the running time of \name. 

\begin{algorithm}[t]
    \caption{\textsc{GCFExplainer}($\phi$, $\mathbb{G}$)}
    \label{alg::gcf}
    {\small
    \begin{algorithmic}[1]
    \STATE $G \gets$ random input graph from $\mathbb{G}$, $N(G) \gets 1$, $\mathbb{S} = \{G\}$
    \FOR{$i \in 1:M$}
    \STATE Let $\epsilon \sim Bernoulli(\tau)$
    \IF{$\epsilon = 0$}
    \FOR{$v \in Neighbors(G)$}
    \STATE Compute $I(v)$ based on \autoref{eq::importance}
    \STATE Compute $p(G, v)$ based on \autoref{eq::transition}
    \ENDFOR
    \STATE $v \gets $ random neighbor of $G$ based on $p(G, v)$ 
    \ELSE
    \STATE $v \gets $ random input graph from $\mathbb{G}$ based on \autoref{eq::teleportation}
    \ENDIF
    \IF{$\phi(v) = 1$}
    \IF{$v \in \mathbb{S}$}
    \STATE $N(v) \gets N(v) + 1$
    \ELSE
    \STATE $\mathbb{S} \gets \mathbb{S} + \{v\}$, $N(v) \gets 1$ 
    \ENDIF
    \ENDIF
    \STATE $G \gets v$
    \ENDFOR
    \STATE $\mathbb{S} \gets $ top $n$ frequently visited counterfactuals in $\mathbb{S}$
    \STATE $\mathbb{C} \gets \emptyset$
    \FOR{$t \in 1:k$}
    \STATE $v \gets \argmax_{v \in \mathbb{S}} \gain(v; \mathbb{C})$
    \STATE $\mathbb{C} \gets \mathbb{C} + \{v\}$
    \ENDFOR
    \RETURN $\mathbb{C}$
\end{algorithmic}}
\end{algorithm}

\section{Experiments}
\label{sec:experiments}

We provide empirical results for the proposed GCFExplanier along with baselines on commonly used graph classification datasets. %
Our code is available at \url{https://github.com/mertkosan/GCFExplainer}. 

\begin{table}[htbp]
\vspace{-0.05in}
\caption{The statistics of the datasets. \label{tab:dataset_stats}}
\vspace{-0.10in}
\centering
{%
\begin{tabular}{ccccc}%
\toprule
& \textbf{NCI1} & \textbf{Mutagenicity} & \textbf{AIDS} & \textbf{Proteins} \\
\midrule
\#Graphs& 3978 & 4308 & 1837 & 1113 \\
\#Nodes& 118714 & 130719 & 28905 & 43471 \\
\#Edges& 128663 & 132707 & 29985 & 81044 \\
\#Node Labels& 10 & 10 & 9 & 3 \\
\bottomrule
\end{tabular}}
\vspace{-0.05in}
\end{table}

\begin{table*}[htbp]
\centering
\caption{Recourse coverage ($\theta=0.1$) and median recourse cost comparison between \name and baselines for a 10-graph global explanation. \name consistently and significantly outperforms all baselines across different datasets. 
}
\vspace{-5pt}
\begin{tabular}{ccccccccc}%
\toprule
& \multicolumn{2}{c}{\textbf{NCI1}} & \multicolumn{2}{c}{\textbf{Mutagenicity}} & \multicolumn{2}{c}{\textbf{AIDS}} & \multicolumn{2}{c}{\textbf{Proteins}} \\
& Coverage & Cost & Coverage & Cost & Coverage & Cost & Coverage & Cost \\
\midrule
\textsc{Ground-Truth} & 
16.54\% & \underline{0.1326} & 
28.96\% & \underline{0.1275} & 
0.41\% & 0.2012 & 
8.47\% & \underline{0.2155} \\
\textsc{RCExplainer} & 
15.22\% & 0.1370 & 
\underline{31.99\%} & 0.1290 & \underline{8.96\%} & \underline{0.1531} & \underline{8.74\%} & 0.2283 \\
\textsc{CFF} & \underline{17.61\%} & 0.1331 & 30.43\% & 0.1327 & 
3.39\% & 0.1669 & 
3.83\% & 0.2557 \\
\name & \textbf{27.85\%} & \textbf{0.1281} & \textbf{37.08\%} & \textbf{0.1135} & \textbf{14.66\%} & \textbf{0.1516} & \textbf{10.93\%} & \textbf{0.1856} \\
\bottomrule
\end{tabular}
\label{tab:quality_methods}

\end{table*}

\begin{figure*}[htbp]
\vspace{-0.1in}
    \centering
    \includegraphics[width=2\columnwidth]{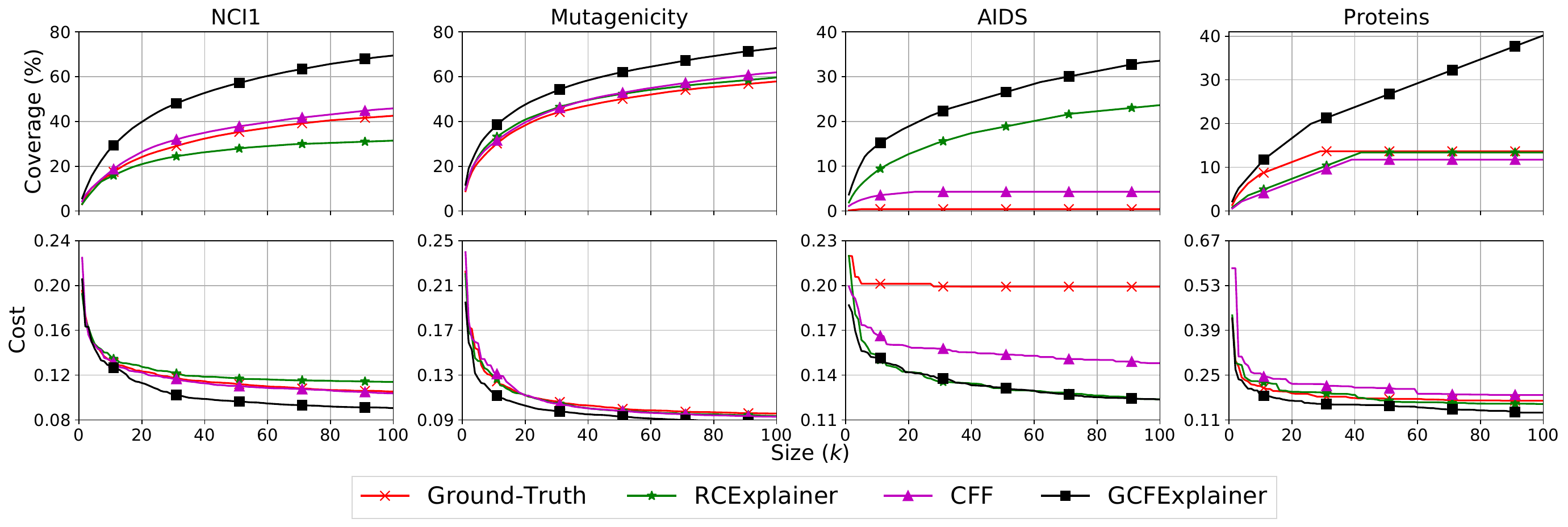}
    \vspace{-0.10in}
    \caption{Coverage and cost performance comparison between \name and baselines based on different counterfactual summary sizes. \name consistently outperforms the baselines across different sizes. 
    \label{fig:coverage_cost_summary_size}}
    \vspace{-0.05in}
\end{figure*}

\subsection{Experimental Settings}

\subsubsection{Datasets} We use four different real-world datasets for graph classification benchmark with their statistics in \autoref{tab:dataset_stats}. Specifically, NCI1 \cite{wale2006comparison}, Mutagenicity \cite{riesen2008iam, kazius2005derivation}, and AIDS \cite{riesen2008iam} are collections of molecules with nodes representing different atoms and edges representing chemical bonds between them. The molecules are classified by whether they are anticancer, mutagenic, and active against HIV, respectively. Proteins \cite{borgwardt2005protein, dobson2003distinguishing} is a collection of proteins classified into enzymes and non-enzymes, with nodes representing secondary structure elements and edges representing structural proximity. 
For all datasets, we filter out graphs containing rare nodes with label frequencies smaller than 50.

\subsubsection{Graph classifier} We follow \cite{vu2020pgm} and train a GNN with 3 convolution layers \cite{kipf2016semi} of embedding dimension 20, a max pooling layer, and a fully connected layer for classification. The model is trained with the Adam optimizer \cite{kingma2014adam} and a learning rate of 0.001 for 1000 epochs. The datasets are split into 80\%/10\%/10\% for training/validation/testing with the model accuracy shown in \autoref{tab:base_gnn_results}.
\looseness=-1

\begin{table}[h]
\caption{Accuracy of the GNN graph classifier. \label{tab:base_gnn_results}}
\vspace{-5pt}
\centering
\begin{tabular}{ccccc}%
\toprule
& \textbf{NCI1} & \textbf{Mutagenicity} & \textbf{AIDS} & \textbf{Proteins} \\
\midrule
Training& 0.8439 & 0.8825 & 0.9980 & 0.7800 \\
Validation& 0.8161 & 0.8302 & 0.9727 & 0.8198 \\
Testing& 0.7809 & 0.8000 & 0.9781 & 0.7297 \\
\bottomrule
\end{tabular}
\vspace{-2mm}
\end{table}

\subsubsection{Baselines} To the best of our knowledge, \name is the first global counterfactual explainer. To validate its effectiveness, we compare it against state-of-the-art local counterfactual explainers combined with the greedy summarization algorithm described in Section~\ref{subsec::summary}. The following local counterfactual generation methods are included in our experiments. 
\begin{itemize}
    \item \textsc{Ground-Truth}: Using graphs belonging to the desired class from the original dataset as local counterfactuals. 
    \item \textsc{RCExplainer} \cite{bajaj2021robust}: Local counterfactual explainer based on the modeling of implicit decision regions of GNNs. 
    \item \textsc{CFF} \cite{tan2022learning}: Local counterfactual explainer based on joint modeling of factual and counterfactual reasoning. 
\end{itemize}

\subsubsection{Explainer settings} We use a distance threshold $\theta$ of 0.05 for training all explainers. Since computing the exact graph edit distance is NP-hard, we apply a state-of-the-art neural approximation algorithm \cite{ranjan2021neural}. %
For \name, we set the teleportation probability $\tau=0.1$ and tune $\alpha$, the weight between individual coverage and gain of coverage, from $\{0, 0.5, 1\}$. A sensitivity analysis is presented in Section~\ref{sec:alpha_tune}. The number of VRRW iterations $M$ is set to 50000, which is enough for convergence as shown in Section~\ref{subsec::convergence} 
For baselines, we tune their hyperparameters to achieve the best local counterfactual rates while maintaining an average distance to input graphs that is smaller than the distance threshold $\theta$.  

\begin{figure*}[ht]
    \centering
    \includegraphics[width=2\columnwidth]{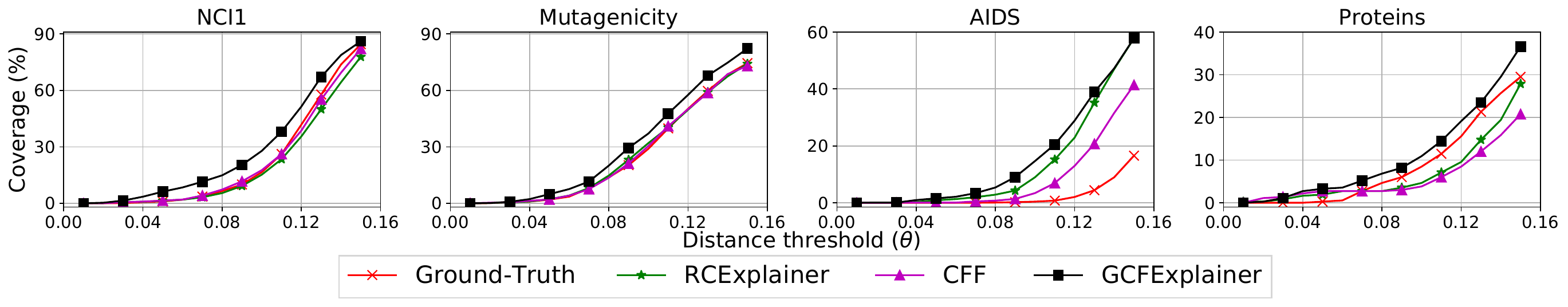}
    \vspace{-0.15in}
    \caption{Recourse coverage comparison between \name and baselines based on different distance threshold values ($\theta$). \name consistently outperforms the baselines across different $\theta$. 
    \label{fig:coverage_based_on_threshold}}
    \vspace{-0.05in}
\end{figure*}

\subsection{Recourse Quality}
\label{subsec::quality}
We start by comparing the recourse quality between \name and baselines. \autoref{tab:quality_methods} shows the recourse coverage with $\theta=0.1$ and median recourse cost of the top 10 counterfactual graphs (i.e., $k=10$). We first notice that the two state-of-the-art local counterfactual explainers have similar performance as \textsc{Ground-Truth}, consistent with our claim that local counterfactual examples from existing methods are not good candidates for a global explanation. The proposed \name, on the other hand, achieves significantly better performance for global recourse quality. Compared to the best baseline, \textsc{RCExplainer}, \name realizes a \textbf{46.9\%} gain in recourse coverage and a \textbf{9.5\%} reduction in recourse cost.

Next, we show the recourse coverage and cost for different sizes of counterfactual summary in \autoref{fig:coverage_cost_summary_size}. As expected, adding more graphs to the recourse representation increases recourse coverage while decreasing recourse cost, at the cost of interpretability. And \name maintains a constant edge over the baselines. 

We also compare the recourse coverage based on different distance thresholds $\theta$, with results shown in \autoref{fig:coverage_based_on_threshold}. While coverage increases for all methods as the threshold increases, \name consistently outperforms the baselines across different sizes.

\subsection{Global Counterfactual Insight}
We have demonstrated the superiority of \name based on various quality metrics for global recourse. Here, we show how \name provides global insights compared to local counterfactual examples. 
\autoref{fig:aids_case_study} illustrates (a) four input undesired graphs with a similar structure from the AIDS dataset, (b) corresponding local counterfactual examples (based on \textsc{RCExplainer} and \textsc{CFF}), and (c) the representative global counterfactual graph from \name covering the input graphs.
Our goal is to understand why the input graphs are inactive against AIDS (undesired) and how to obtain the desired property with minimal changes. 

The local counterfactuals in (b) attribute the classification results to different edges in individual graphs (shown as red dotted lines) and recommend their removal to make input graphs active against HIV. Note that while only two edits are proposed for each individual graph, they appear at different locations, which are hard to generalize for a global view of the model behavior. In contrast, the global counterfactual graph from \name presents a high-level recourse rule. Specifically, the carbon atom with the carbon-oxygen bond is connected to \emph{two} other carbon atoms in the input graphs, making them ketones (with a C=O bond) or ethers (with a C-O bond). On the other hand, the global counterfactual graph highlights a different functional group, aldehyde (shown in blue), to be the key for combating AIDS. In aldehydes, the carbon atom with a carbon-oxygen bond is only connected to \emph{one} other carbon atom, leading to different chemical properties compared to ketones and ethers. Indeed, aldehydes have been shown to be effective HIV protease inhibitors \cite{sarubbi1993peptide}. 

Finally, this case study also demonstrates that counterfactual candidates found by \name are better for global explanation than local counterfactuals. We note that while the graph edit distance between the local counterfactuals and their corresponding input graphs is only 2, they do not cover other similarly structured input graphs (with distance $>5$). Meanwhile, our global counterfactual graph covers all input graphs (with distance $\leq 4$).

\begin{figure}[htbp]
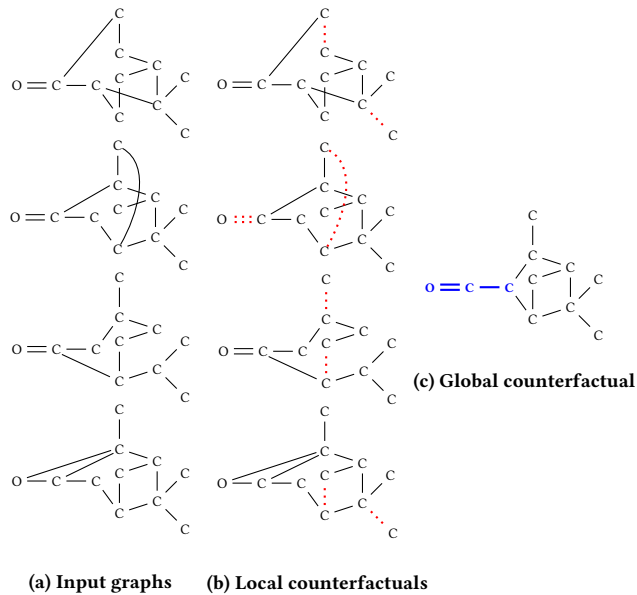


\begin{minipage}{.15\textwidth}%
\centering
  \tiny\chemfig{C(-[:-18,1.5,,,,,,])([:180]-C(-[:50,2.2,,,,,,])=O)*5(-C([:90,1.0]-C?)-[,,,,thick,dashdotted,,brown,draw=none]C([:45]-C)([:-45]-C)-C?-C([:90]-C)-[,,,,thick,dashdotted,,brown,draw=none])}
  \qquad     
  \tiny\chemfig{C([:180]-C(-[:30,1.6,,,,,,])=O)*5(-@{b}C([:90,1.0,,,thick,dashdotted,,brown,draw=none]-C?)-C([:45]-C)([:315]-C)-C?-C([:90]-@{a}C)-[,,,,thick,dashdotted,,brown,draw=none])}
  \chemmove{\draw[-](a)..controls +(-30:3em) and +(60:3em)..(b);}
  \qquad
  \tiny
    \chemfig{C([:180]-C=O)*5(-[,,,,thick,dashdotted,,brown,draw=none]C(-[:155,1.6,,,,,,])([:90,1.0]-C?)-C([:45]-C)([:315]-C)-[,,,,thick,dashdotted,,brown,draw=none]C?-C([:90]-C)-)}
  \qquad  
  \tiny
    \chemfig{C([:180]-C(-[:26,1.6,,,,,,])-O(-[:18,2.5,,,,,,]))*5(-C([:90,1.0]-C?)-C([:45]-C)([:315]-C)-C?-C([:90]-C)-[,,,,thick,dashdotted,,brown,draw=none])}
 \\
  \vspace{5mm}
  \scalebox{1.5}{\textbf{(a) Input graphs}}
\end{minipage}
\begin{minipage}{.15\textwidth}%
\centering
   \tiny\chemfig{C(-[:-18,1.4,,,,,,])([:180]-C(-[:50,2.2,,,,,,])=O)*5(-C([:90,1.0]-C?)-[,,,,draw=none]C([:45]-C)(-[:315,1.1,,,thick,dotted,,red]C)-C?-C(-[:90,,,,thick,dotted,,red]C))}
   \qquad 
  \tiny\chemfig{C([:180]-C(-[:30,1.6,,,,,,])=[,,,,thick,dotted,,red]O)*5(-@{b}C([:90,1.0,,,thick,dashdotted,,brown,draw=none]-C?)-C([:45]-C)([:315]-C)-C?-C([:90]-@{a}C)-[,,,,thick,dashdotted,,brown,draw=none])}
  \chemmove{\draw[-,thick,dotted,red](a)..controls +(-30:3em) and +(60:3em)..(b);}
  \qquad 
    \tiny\chemfig{C([:180]-C=O)*5(-[,,,,thick,dashdotted,,brown,draw=none]C(-[:155,1.6,,,,,,])(-[:90,1.0,,,thick,dotted,,red]C?)-C([:45]-C)([:315]-C)-[,,,,thick,dashdotted,,brown,draw=none]C?-C(-[:90,,,,thick,dotted,,red]C)-)}
  \qquad 
    \tiny\chemfig{C([:180]-C(-[:26,1.6,,,,,,])-O(-[:18,2.5,,,,,,]))*5(-C(-[:90,1.0,,,thick,dotted,,red]C?)-C([:45]-C)(-[:315,1.1,,,thick,dotted,,red]C)-C?-C([:90]-C)-[,,,,thick,dashdotted,,brown,draw=none])}
 \\
  \vspace{4.5mm}
  \scalebox{1.5}{\textbf{(b) Local counterfactuals}}
\end{minipage}
\begin{minipage}{.15\textwidth}
\centering
    \tiny\chemfig{\textcolor{blue}{\textbf{C}}([:180]-[,,,,thick,,,blue]\textcolor{blue}{\textbf{C}}=[,,,,thick,,,blue]\textcolor{blue}{\textbf{O}})*5(-C(-[:90,1.0]C?)-C([:45]-C)([:315]-C)-C?-C([:90]-C)-)} \\
    \vspace{4mm}
  \scalebox{1.5}{\textbf{(c) Global counterfactual}}
\end{minipage}
\vspace{-0.05in}
\caption{Illustration of global and local counterfactual explanations for the AIDS dataset. The global counterfactual graph (c) presents a high-level recourse rule---changing ketones and ethers into aldehydes (shown in blue)---to combat HIV, while the edge removals (shown in red) recommended by local counterfactual examples (b) are hard to generalize. 
}

\label{fig:aids_case_study}
\end{figure}

\subsection{Ablation Study}
We then conduct an ablation study to investigate the effectiveness of \name components. We consider three alternatives:
\begin{itemize}
    \item \textsc{\name-NVR}: no vertex-reinforcement ($N(v)=1$)
    \item \textsc{\name-NIF}: no importance function ($I(v)=1$)
    \item \textsc{\name-NDT}: no dynamic teleportation ($p_\tau(G)=1/|\mathbb{G}|$)%
    \looseness=-1
\end{itemize}
The coverage results are shown in \autoref{tab:ablation}. We observe decreased performance when any of \name components is absent. 
\looseness=-1

\begin{table}[h]
\vspace{-0.05in
}
\small
\caption{Ablation study results based on recourse coverage.  \label{tab:ablation}}
\centering
\setlength{\tabcolsep}{4pt}
\begin{tabular}{ccccc}%
\toprule
& \textbf{NCI1} & \textbf{Mutagenicity} & \textbf{AIDS} & \textbf{Proteins} \\
\midrule
\name-NVR & 24.56\% & 35.44\% & 11.33\% & 8.56\% \\
\name-NIF & 13.29\% & 29.16\% & 4.54\% & 6.83\% \\
\name-NDT & 27.34\% & 36.35\% & 14.05\% & 9.28\% \\
\name & \textbf{27.85\%} & \textbf{37.08\%} & \textbf{14.66\%} & \textbf{10.93\%}  \\
\bottomrule
\end{tabular}
\end{table}

\subsection{Convergence Analysis}
\label{subsec::convergence}
In this subsection, we show the empirical convergence of VRRW based on the mutagenicity dataset in \autoref{fig:convergence_mutagenicity}. We observe that the coverage performance for different summary sizes starts to converge after 15000 iterations and fully converges after 50000 iterations, which is the number we applied in our experiments. 

\begin{figure}[ht]
\vspace{-0.05in}
    \centering
    \includegraphics[width=0.4\textwidth]{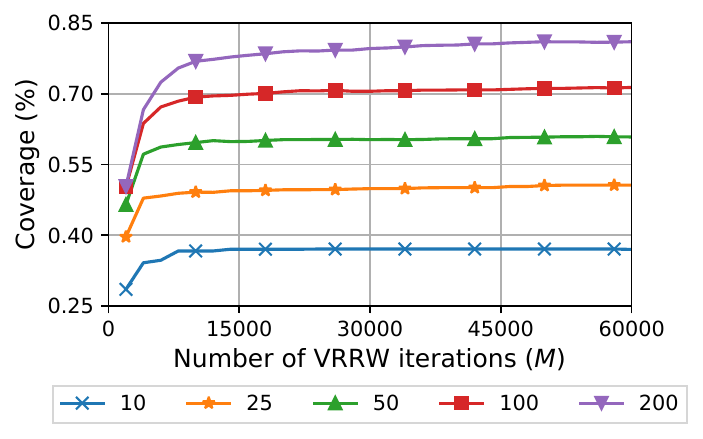}
    \vspace{-0.15in}
    \caption{Convergence of VRRW for the mutagenicity dataset based on recourse coverage with different summary sizes. VRRW fully converges after $M=50000$ iterations. 
    \label{fig:convergence_mutagenicity}}
    \vspace{-0.1in}
\end{figure}

\subsection{Sensitivity Analysis}
\label{sec:alpha_tune}
The only hyperparameter of \name we tune is $\alpha$ in \autoref{eq::importance} that weights the individual coverage and gain of coverage for the importance function. \autoref{tab:alpha_tune} shows the results based on different $\alpha$. While \name outperforms baselines with all different $\alpha$, we observe that individual coverage works better for NCI1 and gain of cumulative coverage works better for other datasets.  
\begin{table}[ht]
\caption{Sensitivity analysis on $\alpha$, the weight between individual coverage and gain of coverage in the importance function. %
\label{tab:alpha_tune}}
\vspace{-0.15in}
\centering
\begin{tabular}{ccccc}%
\toprule
& \textbf{NCI1} & \textbf{Mutagenicity} & \textbf{AIDS} & \textbf{Proteins} \\
\midrule
$\alpha = 0.0$ & \textbf{27.85\%} & 36.87\% & 12.83\% & 10.11\% \\
$\alpha = 0.5$ & 27.50\% & 36.59\% & \textbf{14.66\%} & 10.38\% \\
$\alpha = 1.0$ & 22.27\% & \textbf{37.08\%} & 13.99\% & \textbf{10.93\%} \\
\bottomrule
\end{tabular}
\end{table}

\subsection{Running Time}
\autoref{tab:running_times} summarizes the running times of generating counterfactual candidates based on different methods. %
\name has a competitive running time albeit exploring more counterfactual graphs in the process. We also include results for \textsc{\name-S} which samples a maximum of 10000 neighbors for computing the importance at each step. It achieves better running time at a negligible cost of 3.3\% performance loss on average. Finally, summarizing the counterfactual candidates takes less than a second for all methods. 
\looseness=-1

\begin{table}[ht]
\small
\caption{Counterfactual candidates generation time comparison. \name(-S) has competitive running time albeit exploring more counterfactual graphs. 
\label{tab:running_times}}
\centering
\begin{tabular}{ccccc}%
\toprule
& \textbf{NCI1} & \textbf{Mutagenicity} & \textbf{AIDS} & \textbf{Proteins} \\
\midrule
\textsc{RCExplainer} & 30454 & 52549 & 29047 & 8444 \\
\textsc{CFF} & 22794 & 31749 & 21296 & \textbf{6412} \\
\name & \underline{19817} & \underline{24006} & \underline{2615} & 19246 \\ %
\textsc{\name-S} & \textbf{19365} & \textbf{18798} & \textbf{2539} & \underline{7429} \\
\bottomrule
\end{tabular}
\end{table}

\section{Related Work}

\textbf{Explanations for Graph Neural Networks.} 
There is much research \cite{ying2019gnnexplainer,luo2020parameterized,vu2020pgm,xgnn_kdd20} on explaining graph neural networks (GNNs). The first proposed method, GNNExplainer \cite{ying2019gnnexplainer}, finds the explanatory subgraph and sub-features by maximizing the mutual information between the original prediction and the prediction based on the subgraph and sub-features. 
Later, PGExplainer \cite{luo2020parameterized} provides an inductive framework %
that extracts GNN node embeddings and learns to map  embedding pairs to the probability of edge existence in the explanatory weighted subgraph. %
PGMExplainer \cite{vu2020pgm} %
builds a probabilistic explanation model that learns new predictions from perturbed node features, performs variable selection using Markov blanket of variables, and then produces a Bayesian network via structure learning. In XGNN \cite{xgnn_kdd20}, the authors %
find model-level explanations %
by a graph generation module that outputs a sequence of edges using reinforcement learning. 
These explanation methods focus on \emph{factual} reasoning while the goal of our work is to provide a global \emph{counterfactual} explanation for GNNs. 
\looseness=-1

\noindent\textbf{Counterfactual Explanations.} 
Recently, there are several attempts to have explanations of graph neural networks (GNNs) via counterfactual reasoning \cite{lucic2021cf,bajaj2021robust,abrate2021counterfactual,tan2022learning}.
One of the earlier methods, 
CF-GNNExplainer \cite{lucic2021cf}, provides counterfactual explanations in terms of a learnable perturbed adjacency matrix that leads to the flipping of classifier prediction for a node. %
On the other hand, RCExplainer \cite{bajaj2021robust} aims to find a robust subset of edges 
whose removal changes the prediction of the remaining graph by modeling the implicit decision regions based on GNN graph embeddings.
In \cite{abrate2021counterfactual}, the authors investigate counterfactual explanations for a more specific class of graphs---the brain networks---that share the same set of nodes by greedily adding or removing edges using a heuristic. 
More recently, the authors of CFF \cite{tan2022learning} argue that a good explanation for GNNs should consider both factual and counterfactual reasoning and they explicitly incorporate those objective functions when searching for the best explanatory subgraphs and sub-features. 
Counterfactual reasoning has also been applied for link prediction \cite{zhao2022learning}.
All the above methods produce \textit{local} counterfactual examples while our work aims to provide a \emph{global} explanation in terms of a summary of representative counterfactual graphs.
\vspace{-1mm}
\section{Conclusion}
We have proposed \name, the first global counterfactual explainer for graph classification. Compared to local explainers, \name provides a high-level picture of the model behavior and effective global recourse rules. We hope that our work will not only deepen our understanding of graph neural networks but also build a bridge for experts from other domains to leverage deep learning models for high-stakes decision-making. 
\begin{acks}
This work is partially funded by NSF via grant IIS 1817046. The authors thank Sevgi Kosan and Ilhan Kosan for their helpful comments on the chemical properties in the global counterfactual analysis.
\end{acks}

\clearpage
\bibliographystyle{ACM-Reference-Format}
\bibliography{main}


\begin{thebibliography}{48}


\ifx \showCODEN    \undefined \def \showCODEN     #1{\unskip}     \fi
\ifx \showDOI      \undefined \def \showDOI       #1{#1}\fi
\ifx \showISBNx    \undefined \def \showISBNx     #1{\unskip}     \fi
\ifx \showISBNxiii \undefined \def \showISBNxiii  #1{\unskip}     \fi
\ifx \showISSN     \undefined \def \showISSN      #1{\unskip}     \fi
\ifx \showLCCN     \undefined \def \showLCCN      #1{\unskip}     \fi
\ifx \shownote     \undefined \def \shownote      #1{#1}          \fi
\ifx \showarticletitle \undefined \def \showarticletitle #1{#1}   \fi
\ifx \showURL      \undefined \def \showURL       {\relax}        \fi
\providecommand\bibfield[2]{#2}
\providecommand\bibinfo[2]{#2}
\providecommand\natexlab[1]{#1}
\providecommand\showeprint[2][]{arXiv:#2}

\bibitem[\protect\citeauthoryear{Abrate and Bonchi}{Abrate and Bonchi}{2021}]%
        {abrate2021counterfactual}
\bibfield{author}{\bibinfo{person}{Carlo Abrate} {and}
  \bibinfo{person}{Francesco Bonchi}.} \bibinfo{year}{2021}\natexlab{}.
\newblock \showarticletitle{Counterfactual graphs for explainable
  classification of brain networks}. In \bibinfo{booktitle}{\emph{SIGKDD}}.
\newblock


\bibitem[\protect\citeauthoryear{Bajaj, Chu, Xue, Pei, Wang, Lam, and
  Zhang}{Bajaj et~al\mbox{.}}{2021}]%
        {bajaj2021robust}
\bibfield{author}{\bibinfo{person}{Mohit Bajaj}, \bibinfo{person}{Lingyang
  Chu}, \bibinfo{person}{Zi~Yu Xue}, \bibinfo{person}{Jian Pei},
  \bibinfo{person}{Lanjun Wang}, \bibinfo{person}{Peter Cho-Ho Lam}, {and}
  \bibinfo{person}{Yong Zhang}.} \bibinfo{year}{2021}\natexlab{}.
\newblock \showarticletitle{Robust Counterfactual Explanations on Graph Neural
  Networks}. In \bibinfo{booktitle}{\emph{NeurIPS}}.
\newblock


\bibitem[\protect\citeauthoryear{Bhattoo, Ranu, and Krishnan}{Bhattoo
  et~al\mbox{.}}{2022}]%
        {lgnn}
\bibfield{author}{\bibinfo{person}{Ravinder Bhattoo}, \bibinfo{person}{Sayan
  Ranu}, {and} \bibinfo{person}{NM Krishnan}.} \bibinfo{year}{2022}\natexlab{}.
\newblock \showarticletitle{Learning Articulated Rigid Body Dynamics with
  Lagrangian Graph Neural Network}. In \bibinfo{booktitle}{\emph{NeurIPS}}.
\newblock


\bibitem[\protect\citeauthoryear{Borgwardt, Schraudolph, and
  Vishwanathan}{Borgwardt et~al\mbox{.}}{2006}]%
        {borgwardt2006fast}
\bibfield{author}{\bibinfo{person}{Karsten Borgwardt}, \bibinfo{person}{Nicol
  Schraudolph}, {and} \bibinfo{person}{SVN Vishwanathan}.}
  \bibinfo{year}{2006}\natexlab{}.
\newblock \showarticletitle{Fast computation of graph kernels}. In
  \bibinfo{booktitle}{\emph{NeurIPS}}.
\newblock


\bibitem[\protect\citeauthoryear{Borgwardt, Ong, Sch{\"o}nauer, Vishwanathan,
  Smola, and Kriegel}{Borgwardt et~al\mbox{.}}{2005}]%
        {borgwardt2005protein}
\bibfield{author}{\bibinfo{person}{Karsten~M Borgwardt},
  \bibinfo{person}{Cheng~Soon Ong}, \bibinfo{person}{Stefan Sch{\"o}nauer},
  \bibinfo{person}{SVN Vishwanathan}, \bibinfo{person}{Alex~J Smola}, {and}
  \bibinfo{person}{Hans-Peter Kriegel}.} \bibinfo{year}{2005}\natexlab{}.
\newblock \showarticletitle{Protein function prediction via graph kernels}.
\newblock \bibinfo{journal}{\emph{Bioinformatics}} \bibinfo{volume}{21},
  \bibinfo{number}{suppl\_1} (\bibinfo{year}{2005}), \bibinfo{pages}{i47--i56}.
\newblock


\bibitem[\protect\citeauthoryear{Costa and De~Grave}{Costa and
  De~Grave}{2010}]%
        {costa2010fast}
\bibfield{author}{\bibinfo{person}{Fabrizio Costa} {and} \bibinfo{person}{Kurt
  De~Grave}.} \bibinfo{year}{2010}\natexlab{}.
\newblock \showarticletitle{Fast neighborhood subgraph pairwise distance
  kernel}. In \bibinfo{booktitle}{\emph{ICML}}.
\newblock


\bibitem[\protect\citeauthoryear{Dobson and Doig}{Dobson and Doig}{2003}]%
        {dobson2003distinguishing}
\bibfield{author}{\bibinfo{person}{Paul~D Dobson} {and}
  \bibinfo{person}{Andrew~J Doig}.} \bibinfo{year}{2003}\natexlab{}.
\newblock \showarticletitle{Distinguishing enzyme structures from non-enzymes
  without alignments}.
\newblock \bibinfo{journal}{\emph{Journal of molecular biology}}
  \bibinfo{volume}{330}, \bibinfo{number}{4} (\bibinfo{year}{2003}),
  \bibinfo{pages}{771--783}.
\newblock


\bibitem[\protect\citeauthoryear{Gupta, Manchanda, Bedathur, and Ranu}{Gupta
  et~al\mbox{.}}{2022}]%
        {tigger}
\bibfield{author}{\bibinfo{person}{Shubham Gupta}, \bibinfo{person}{Sahil
  Manchanda}, \bibinfo{person}{Srikanta Bedathur}, {and} \bibinfo{person}{Sayan
  Ranu}.} \bibinfo{year}{2022}\natexlab{}.
\newblock \showarticletitle{{TIGGER:} Scalable Generative Modelling for
  Temporal Interaction Graphs}. In \bibinfo{booktitle}{\emph{Thirty-Sixth
  {AAAI} Conference on Artificial Intelligence, {AAAI} 2022, Thirty-Fourth
  Conference on Innovative Applications of Artificial Intelligence, {IAAI}
  2022, The Twelveth Symposium on Educational Advances in Artificial
  Intelligence, {EAAI} 2022 Virtual Event, February 22 - March 1, 2022}}.
  \bibinfo{publisher}{{AAAI} Press}, \bibinfo{pages}{6819--6828}.
\newblock
\urldef\tempurl%
\url{https://ojs.aaai.org/index.php/AAAI/article/view/20638}
\showURL{%
\tempurl}


\bibitem[\protect\citeauthoryear{Hamilton, Ying, and Leskovec}{Hamilton
  et~al\mbox{.}}{2017}]%
        {hamilton2017inductive}
\bibfield{author}{\bibinfo{person}{Will Hamilton}, \bibinfo{person}{Zhitao
  Ying}, {and} \bibinfo{person}{Jure Leskovec}.}
  \bibinfo{year}{2017}\natexlab{}.
\newblock \showarticletitle{Inductive representation learning on large graphs}.
\newblock \bibinfo{journal}{\emph{NeurIPS}} (\bibinfo{year}{2017}).
\newblock


\bibitem[\protect\citeauthoryear{Huang, Silva, and Singh}{Huang
  et~al\mbox{.}}{2021}]%
        {huang2021broader}
\bibfield{author}{\bibinfo{person}{Zexi Huang}, \bibinfo{person}{Arlei Silva},
  {and} \bibinfo{person}{Ambuj Singh}.} \bibinfo{year}{2021}\natexlab{}.
\newblock \showarticletitle{A broader picture of random-walk based graph
  embedding}. In \bibinfo{booktitle}{\emph{SIGKDD}}.
\newblock


\bibitem[\protect\citeauthoryear{Huang, Silva, and Singh}{Huang
  et~al\mbox{.}}{2022}]%
        {huang2022pole}
\bibfield{author}{\bibinfo{person}{Zexi Huang}, \bibinfo{person}{Arlei Silva},
  {and} \bibinfo{person}{Ambuj Singh}.} \bibinfo{year}{2022}\natexlab{}.
\newblock \showarticletitle{POLE: Polarized Embedding for Signed Networks}. In
  \bibinfo{booktitle}{\emph{WSDM}}.
\newblock


\bibitem[\protect\citeauthoryear{Jiang, Li, Zhang, Wang, Wang, Yuan, and
  Wei}{Jiang et~al\mbox{.}}{2020}]%
        {jiang2020drug}
\bibfield{author}{\bibinfo{person}{Mingjian Jiang}, \bibinfo{person}{Zhen Li},
  \bibinfo{person}{Shugang Zhang}, \bibinfo{person}{Shuang Wang},
  \bibinfo{person}{Xiaofeng Wang}, \bibinfo{person}{Qing Yuan}, {and}
  \bibinfo{person}{Zhiqiang Wei}.} \bibinfo{year}{2020}\natexlab{}.
\newblock \showarticletitle{Drug--target affinity prediction using graph neural
  network and contact maps}.
\newblock \bibinfo{journal}{\emph{RSC advances}} \bibinfo{volume}{10},
  \bibinfo{number}{35} (\bibinfo{year}{2020}), \bibinfo{pages}{20701--20712}.
\newblock


\bibitem[\protect\citeauthoryear{Kazius, McGuire, and Bursi}{Kazius
  et~al\mbox{.}}{2005}]%
        {kazius2005derivation}
\bibfield{author}{\bibinfo{person}{Jeroen Kazius}, \bibinfo{person}{Ross
  McGuire}, {and} \bibinfo{person}{Roberta Bursi}.}
  \bibinfo{year}{2005}\natexlab{}.
\newblock \showarticletitle{Derivation and validation of toxicophores for
  mutagenicity prediction}.
\newblock \bibinfo{journal}{\emph{Journal of medicinal chemistry}}
  \bibinfo{volume}{48}, \bibinfo{number}{1} (\bibinfo{year}{2005}),
  \bibinfo{pages}{312--320}.
\newblock


\bibitem[\protect\citeauthoryear{Kingma and Ba}{Kingma and Ba}{2014}]%
        {kingma2014adam}
\bibfield{author}{\bibinfo{person}{Diederik~P Kingma} {and}
  \bibinfo{person}{Jimmy Ba}.} \bibinfo{year}{2014}\natexlab{}.
\newblock \showarticletitle{Adam: A method for stochastic optimization}.
\newblock \bibinfo{journal}{\emph{arXiv:1412.6980}} (\bibinfo{year}{2014}).
\newblock


\bibitem[\protect\citeauthoryear{Kipf and Welling}{Kipf and Welling}{2017}]%
        {kipf2016semi}
\bibfield{author}{\bibinfo{person}{Thomas~N Kipf} {and} \bibinfo{person}{Max
  Welling}.} \bibinfo{year}{2017}\natexlab{}.
\newblock \showarticletitle{Semi-supervised classification with graph
  convolutional networks}. In \bibinfo{booktitle}{\emph{ICLR}}.
\newblock


\bibitem[\protect\citeauthoryear{Klicpera, Bojchevski, and
  G{\"u}nnemann}{Klicpera et~al\mbox{.}}{2018}]%
        {klicpera2018predict}
\bibfield{author}{\bibinfo{person}{Johannes Klicpera},
  \bibinfo{person}{Aleksandar Bojchevski}, {and} \bibinfo{person}{Stephan
  G{\"u}nnemann}.} \bibinfo{year}{2018}\natexlab{}.
\newblock \showarticletitle{Predict then propagate: Graph neural networks meet
  personalized pagerank}. In \bibinfo{booktitle}{\emph{ICLR}}.
\newblock


\bibitem[\protect\citeauthoryear{Kosan, Silva, Medya, Uzzi, and Singh}{Kosan
  et~al\mbox{.}}{2021}]%
        {kosan2021event}
\bibfield{author}{\bibinfo{person}{Mert Kosan}, \bibinfo{person}{Arlei Silva},
  \bibinfo{person}{Sourav Medya}, \bibinfo{person}{Brian Uzzi}, {and}
  \bibinfo{person}{Ambuj Singh}.} \bibinfo{year}{2021}\natexlab{}.
\newblock \showarticletitle{Event detection on dynamic graphs}.
\newblock \bibinfo{journal}{\emph{arXiv preprint arXiv:2110.12148}}
  (\bibinfo{year}{2021}).
\newblock


\bibitem[\protect\citeauthoryear{Liang and Zhao}{Liang and Zhao}{2017}]%
        {ged1}
\bibfield{author}{\bibinfo{person}{Yongjiang Liang} {and}
  \bibinfo{person}{Peixiang Zhao}.} \bibinfo{year}{2017}\natexlab{}.
\newblock \showarticletitle{Similarity Search in Graph Databases: A
  Multi-Layered Indexing Approach}. In \bibinfo{booktitle}{\emph{ICDE}}.
  \bibinfo{pages}{783--794}.
\newblock


\bibitem[\protect\citeauthoryear{Lucic, Ter~Hoeve, Tolomei, De~Rijke, and
  Silvestri}{Lucic et~al\mbox{.}}{2022}]%
        {lucic2021cf}
\bibfield{author}{\bibinfo{person}{Ana Lucic}, \bibinfo{person}{Maartje~A
  Ter~Hoeve}, \bibinfo{person}{Gabriele Tolomei}, \bibinfo{person}{Maarten
  De~Rijke}, {and} \bibinfo{person}{Fabrizio Silvestri}.}
  \bibinfo{year}{2022}\natexlab{}.
\newblock \showarticletitle{Cf-gnnexplainer: Counterfactual explanations for
  graph neural networks}. In \bibinfo{booktitle}{\emph{AISTATS}}.
\newblock


\bibitem[\protect\citeauthoryear{Luo, Cheng, Xu, Yu, Zong, Chen, and Zhang}{Luo
  et~al\mbox{.}}{2020}]%
        {luo2020parameterized}
\bibfield{author}{\bibinfo{person}{Dongsheng Luo}, \bibinfo{person}{Wei Cheng},
  \bibinfo{person}{Dongkuan Xu}, \bibinfo{person}{Wenchao Yu},
  \bibinfo{person}{Bo Zong}, \bibinfo{person}{Haifeng Chen}, {and}
  \bibinfo{person}{Xiang Zhang}.} \bibinfo{year}{2020}\natexlab{}.
\newblock \showarticletitle{Parameterized explainer for graph neural network}.
  In \bibinfo{booktitle}{\emph{NeurIPS}}.
\newblock


\bibitem[\protect\citeauthoryear{Manchanda, Mittal, Dhawan, Medya, Ranu, and
  Singh}{Manchanda et~al\mbox{.}}{2020}]%
        {gcomb}
\bibfield{author}{\bibinfo{person}{Sahil Manchanda}, \bibinfo{person}{Akash
  Mittal}, \bibinfo{person}{Anuj Dhawan}, \bibinfo{person}{Sourav Medya},
  \bibinfo{person}{Sayan Ranu}, {and} \bibinfo{person}{Ambuj Singh}.}
  \bibinfo{year}{2020}\natexlab{}.
\newblock \showarticletitle{Gcomb: Learning budget-constrained combinatorial
  algorithms over billion-sized graphs}.
\newblock \bibinfo{journal}{\emph{NeurIPS}}  \bibinfo{volume}{33}
  (\bibinfo{year}{2020}), \bibinfo{pages}{20000--20011}.
\newblock


\bibitem[\protect\citeauthoryear{Medya, Rasoolinejad, Yang, and Uzzi}{Medya
  et~al\mbox{.}}{2022}]%
        {medya_earnings_22}
\bibfield{author}{\bibinfo{person}{Sourav Medya}, \bibinfo{person}{Mohammad
  Rasoolinejad}, \bibinfo{person}{Yang Yang}, {and} \bibinfo{person}{Brian
  Uzzi}.} \bibinfo{year}{2022}\natexlab{}.
\newblock \showarticletitle{An Exploratory Study of Stock Price Movements from
  Earnings Calls}. In \bibinfo{booktitle}{\emph{WebConf}}.
\newblock


\bibitem[\protect\citeauthoryear{Mei, Guo, and Radev}{Mei
  et~al\mbox{.}}{2010}]%
        {mei2010divrank}
\bibfield{author}{\bibinfo{person}{Qiaozhu Mei}, \bibinfo{person}{Jian Guo},
  {and} \bibinfo{person}{Dragomir Radev}.} \bibinfo{year}{2010}\natexlab{}.
\newblock \showarticletitle{Divrank: the interplay of prestige and diversity in
  information networks}. In \bibinfo{booktitle}{\emph{SIGKDD}}.
\newblock


\bibitem[\protect\citeauthoryear{Metwally, Agrawal, and Abbadi}{Metwally
  et~al\mbox{.}}{2005}]%
        {metwally2005efficient}
\bibfield{author}{\bibinfo{person}{Ahmed Metwally}, \bibinfo{person}{Divyakant
  Agrawal}, {and} \bibinfo{person}{Amr~El Abbadi}.}
  \bibinfo{year}{2005}\natexlab{}.
\newblock \showarticletitle{Efficient computation of frequent and top-k
  elements in data streams}. In \bibinfo{booktitle}{\emph{ICDT}}.
\newblock


\bibitem[\protect\citeauthoryear{Mirhoseini, Goldie, Yazgan, Jiang, Songhori,
  Wang, Lee, Johnson, Pathak, Bae, Nazi, Pak, Tong, Srinivasa, Hang, Tuncer,
  Babu, Le, Laudon, Ho, Carpenter, and Dean}{Mirhoseini et~al\mbox{.}}{2020}]%
        {chip}
\bibfield{author}{\bibinfo{person}{Azalia Mirhoseini}, \bibinfo{person}{Anna
  Goldie}, \bibinfo{person}{Mustafa Yazgan}, \bibinfo{person}{Joe W.~J. Jiang},
  \bibinfo{person}{Ebrahim~M. Songhori}, \bibinfo{person}{Shen Wang},
  \bibinfo{person}{Young{-}Joon Lee}, \bibinfo{person}{Eric Johnson},
  \bibinfo{person}{Omkar Pathak}, \bibinfo{person}{Sungmin Bae},
  \bibinfo{person}{Azade Nazi}, \bibinfo{person}{Jiwoo Pak},
  \bibinfo{person}{Andy Tong}, \bibinfo{person}{Kavya Srinivasa},
  \bibinfo{person}{William Hang}, \bibinfo{person}{Emre Tuncer},
  \bibinfo{person}{Anand Babu}, \bibinfo{person}{Quoc~V. Le},
  \bibinfo{person}{James Laudon}, \bibinfo{person}{Richard Ho},
  \bibinfo{person}{Roger Carpenter}, {and} \bibinfo{person}{Jeff Dean}.}
  \bibinfo{year}{2020}\natexlab{}.
\newblock \showarticletitle{Chip Placement with Deep Reinforcement Learning}.
\newblock \bibinfo{journal}{\emph{CoRR}}  \bibinfo{volume}{abs/2004.10746}
  (\bibinfo{year}{2020}).
\newblock


\bibitem[\protect\citeauthoryear{Natarajan and Ranu}{Natarajan and
  Ranu}{2016}]%
        {natarajan2016scalable}
\bibfield{author}{\bibinfo{person}{Dheepikaa Natarajan} {and}
  \bibinfo{person}{Sayan Ranu}.} \bibinfo{year}{2016}\natexlab{}.
\newblock \showarticletitle{A scalable and generic framework to mine top-k
  representative subgraph patterns}. In \bibinfo{booktitle}{\emph{ICDM}}.
\newblock


\bibitem[\protect\citeauthoryear{Nishad, Agarwal, Bhattacharya, and
  Ranu}{Nishad et~al\mbox{.}}{2021}]%
        {graphreach}
\bibfield{author}{\bibinfo{person}{Sunil Nishad}, \bibinfo{person}{Shubhangi
  Agarwal}, \bibinfo{person}{Arnab Bhattacharya}, {and} \bibinfo{person}{Sayan
  Ranu}.} \bibinfo{year}{2021}\natexlab{}.
\newblock \showarticletitle{GraphReach: Position-Aware Graph Neural Network
  using Reachability Estimations}. In \bibinfo{booktitle}{\emph{Proceedings of
  the Thirtieth International Joint Conference on Artificial Intelligence,
  {IJCAI} 2021, Virtual Event / Montreal, Canada, 19-27 August 2021}},
  \bibfield{editor}{\bibinfo{person}{Zhi{-}Hua Zhou}} (Ed.).
  \bibinfo{publisher}{ijcai.org}, \bibinfo{pages}{1527--1533}.
\newblock
\urldef\tempurl%
\url{https://doi.org/10.24963/ijcai.2021/211}
\showDOI{\tempurl}


\bibitem[\protect\citeauthoryear{Pemantle}{Pemantle}{1992}]%
        {pemantle1992vertex}
\bibfield{author}{\bibinfo{person}{Robin Pemantle}.}
  \bibinfo{year}{1992}\natexlab{}.
\newblock \showarticletitle{Vertex-reinforced random walk}.
\newblock \bibinfo{journal}{\emph{Probability Theory and Related Fields}}
  \bibinfo{volume}{92}, \bibinfo{number}{1} (\bibinfo{year}{1992}),
  \bibinfo{pages}{117--136}.
\newblock


\bibitem[\protect\citeauthoryear{Perozzi, Al-Rfou, and Skiena}{Perozzi
  et~al\mbox{.}}{2014}]%
        {perozzi2014deepwalk}
\bibfield{author}{\bibinfo{person}{Bryan Perozzi}, \bibinfo{person}{Rami
  Al-Rfou}, {and} \bibinfo{person}{Steven Skiena}.}
  \bibinfo{year}{2014}\natexlab{}.
\newblock \showarticletitle{Deepwalk: Online learning of social
  representations}. In \bibinfo{booktitle}{\emph{SIGKDD}}.
\newblock


\bibitem[\protect\citeauthoryear{Ranjan, Grover, Medya, Chakravarthy,
  Sabharwal, and Ranu}{Ranjan et~al\mbox{.}}{2022}]%
        {ranjan2021neural}
\bibfield{author}{\bibinfo{person}{Rishab Ranjan}, \bibinfo{person}{Siddharth
  Grover}, \bibinfo{person}{Sourav Medya}, \bibinfo{person}{Venkatesan
  Chakravarthy}, \bibinfo{person}{Yogish Sabharwal}, {and}
  \bibinfo{person}{Sayan Ranu}.} \bibinfo{year}{2022}\natexlab{}.
\newblock \showarticletitle{GREED: A Neural Framework for Learning Graph
  Distance Functions}. In \bibinfo{booktitle}{\emph{Annual Conference on Neural
  Information Processing Systems}}.
\newblock


\bibitem[\protect\citeauthoryear{Rawal and Lakkaraju}{Rawal and
  Lakkaraju}{2020}]%
        {rawal2020beyond}
\bibfield{author}{\bibinfo{person}{Kaivalya Rawal} {and}
  \bibinfo{person}{Himabindu Lakkaraju}.} \bibinfo{year}{2020}\natexlab{}.
\newblock \showarticletitle{Beyond individualized recourse: Interpretable and
  interactive summaries of actionable recourses}. In
  \bibinfo{booktitle}{\emph{NeurIPS}}.
\newblock


\bibitem[\protect\citeauthoryear{Riesen and Bunke}{Riesen and Bunke}{2008}]%
        {riesen2008iam}
\bibfield{author}{\bibinfo{person}{Kaspar Riesen} {and} \bibinfo{person}{Horst
  Bunke}.} \bibinfo{year}{2008}\natexlab{}.
\newblock \showarticletitle{IAM graph database repository for graph based
  pattern recognition and machine learning}. In \bibinfo{booktitle}{\emph{Joint
  IAPR International Workshops on Statistical Techniques in Pattern Recognition
  (SPR) and Structural and Syntactic Pattern Recognition (SSPR)}}. Springer,
  \bibinfo{pages}{287--297}.
\newblock


\bibitem[\protect\citeauthoryear{Sanfeliu and Fu}{Sanfeliu and Fu}{1983}]%
        {sanfeliu1983distance}
\bibfield{author}{\bibinfo{person}{Alberto Sanfeliu} {and}
  \bibinfo{person}{King-Sun Fu}.} \bibinfo{year}{1983}\natexlab{}.
\newblock \showarticletitle{A distance measure between attributed relational
  graphs for pattern recognition}.
\newblock \bibinfo{journal}{\emph{IEEE transactions on systems, man, and
  cybernetics}} \bibinfo{number}{3} (\bibinfo{year}{1983}),
  \bibinfo{pages}{353--362}.
\newblock


\bibitem[\protect\citeauthoryear{Sarubbi, Seneci, Angelastro, Peet, Denaro, and
  Islam}{Sarubbi et~al\mbox{.}}{1993}]%
        {sarubbi1993peptide}
\bibfield{author}{\bibinfo{person}{Edoardo Sarubbi},
  \bibinfo{person}{Pier~Fausto Seneci}, \bibinfo{person}{Michael~R Angelastro},
  \bibinfo{person}{Norton~P Peet}, \bibinfo{person}{Maurizio Denaro}, {and}
  \bibinfo{person}{Khalid Islam}.} \bibinfo{year}{1993}\natexlab{}.
\newblock \showarticletitle{Peptide aldehydes as inhibitors of HIV protease}.
\newblock \bibinfo{journal}{\emph{FEBS letters}} \bibinfo{volume}{319},
  \bibinfo{number}{3} (\bibinfo{year}{1993}), \bibinfo{pages}{253--256}.
\newblock


\bibitem[\protect\citeauthoryear{Shervashidze, Schweitzer, Van~Leeuwen,
  Mehlhorn, and Borgwardt}{Shervashidze et~al\mbox{.}}{2011}]%
        {shervashidze2011weisfeiler}
\bibfield{author}{\bibinfo{person}{Nino Shervashidze}, \bibinfo{person}{Pascal
  Schweitzer}, \bibinfo{person}{Erik~Jan Van~Leeuwen}, \bibinfo{person}{Kurt
  Mehlhorn}, {and} \bibinfo{person}{Karsten~M Borgwardt}.}
  \bibinfo{year}{2011}\natexlab{}.
\newblock \showarticletitle{Weisfeiler-lehman graph kernels.}
\newblock \bibinfo{journal}{\emph{JMLR}} \bibinfo{volume}{12},
  \bibinfo{number}{9} (\bibinfo{year}{2011}).
\newblock


\bibitem[\protect\citeauthoryear{Tan, Geng, Fu, Ge, Xu, Li, and Zhang}{Tan
  et~al\mbox{.}}{2022}]%
        {tan2022learning}
\bibfield{author}{\bibinfo{person}{Juntao Tan}, \bibinfo{person}{Shijie Geng},
  \bibinfo{person}{Zuohui Fu}, \bibinfo{person}{Yingqiang Ge},
  \bibinfo{person}{Shuyuan Xu}, \bibinfo{person}{Yunqi Li}, {and}
  \bibinfo{person}{Yongfeng Zhang}.} \bibinfo{year}{2022}\natexlab{}.
\newblock \showarticletitle{Learning and evaluating graph neural network
  explanations based on counterfactual and factual reasoning}. In
  \bibinfo{booktitle}{\emph{WebConf}}.
\newblock


\bibitem[\protect\citeauthoryear{Thangamuthu, Kumar, Bishnoi, Bhattoo,
  Krishnan, and Ranu}{Thangamuthu et~al\mbox{.}}{2022}]%
        {lgnn_benchmarking}
\bibfield{author}{\bibinfo{person}{Abishek Thangamuthu},
  \bibinfo{person}{Gunjan Kumar}, \bibinfo{person}{Suresh Bishnoi},
  \bibinfo{person}{Ravinder Bhattoo}, \bibinfo{person}{N~M~Anoop Krishnan},
  {and} \bibinfo{person}{Sayan Ranu}.} \bibinfo{year}{2022}\natexlab{}.
\newblock \showarticletitle{Unravelling the Performance of Physics-informed
  Graph Neural Networks for Dynamical Systems}. In
  \bibinfo{booktitle}{\emph{NeurIPS}}.
\newblock
\urldef\tempurl%
\url{https://openreview.net/forum?id=tXEe-Ew_ikh}
\showURL{%
\tempurl}


\bibitem[\protect\citeauthoryear{Veli{\v{c}}kovi{\'c}, Cucurull, Casanova,
  Romero, Li{\`o}, and Bengio}{Veli{\v{c}}kovi{\'c} et~al\mbox{.}}{2018}]%
        {velivckovic2018graph}
\bibfield{author}{\bibinfo{person}{Petar Veli{\v{c}}kovi{\'c}},
  \bibinfo{person}{Guillem Cucurull}, \bibinfo{person}{Arantxa Casanova},
  \bibinfo{person}{Adriana Romero}, \bibinfo{person}{Pietro Li{\`o}}, {and}
  \bibinfo{person}{Yoshua Bengio}.} \bibinfo{year}{2018}\natexlab{}.
\newblock \showarticletitle{Graph Attention Networks}. In
  \bibinfo{booktitle}{\emph{ICLR}}.
\newblock


\bibitem[\protect\citeauthoryear{Voigt and Von~dem Bussche}{Voigt and Von~dem
  Bussche}{2017}]%
        {voigt2017eu}
\bibfield{author}{\bibinfo{person}{Paul Voigt} {and} \bibinfo{person}{Axel
  Von~dem Bussche}.} \bibinfo{year}{2017}\natexlab{}.
\newblock \showarticletitle{The eu general data protection regulation (gdpr)}.
\newblock \bibinfo{journal}{\emph{A Practical Guide, 1st Ed., Cham: Springer
  International Publishing}} \bibinfo{volume}{10}, \bibinfo{number}{3152676}
  (\bibinfo{year}{2017}), \bibinfo{pages}{10--5555}.
\newblock


\bibitem[\protect\citeauthoryear{Vu and Thai}{Vu and Thai}{2020}]%
        {vu2020pgm}
\bibfield{author}{\bibinfo{person}{Minh Vu} {and} \bibinfo{person}{My~T Thai}.}
  \bibinfo{year}{2020}\natexlab{}.
\newblock \showarticletitle{Pgm-explainer: Probabilistic graphical model
  explanations for graph neural networks}. In
  \bibinfo{booktitle}{\emph{NeurIPS}}.
\newblock


\bibitem[\protect\citeauthoryear{Wale and Karypis}{Wale and Karypis}{2006}]%
        {wale2006comparison}
\bibfield{author}{\bibinfo{person}{Nikil Wale} {and} \bibinfo{person}{George
  Karypis}.} \bibinfo{year}{2006}\natexlab{}.
\newblock \showarticletitle{Comparison of Descriptor Spaces for Chemical
  Compound Retrieval and Classification}. In \bibinfo{booktitle}{\emph{ICDM}}.
\newblock


\bibitem[\protect\citeauthoryear{Wang, Feng, and Ding}{Wang
  et~al\mbox{.}}{2022}]%
        {wang2022qgtc}
\bibfield{author}{\bibinfo{person}{Yuke Wang}, \bibinfo{person}{Boyuan Feng},
  {and} \bibinfo{person}{Yufei Ding}.} \bibinfo{year}{2022}\natexlab{}.
\newblock \showarticletitle{QGTC: accelerating quantized graph neural networks
  via GPU tensor core}. In \bibinfo{booktitle}{\emph{PPoPP}}.
\newblock


\bibitem[\protect\citeauthoryear{Wang, Feng, Li, Li, Deng, Xie, and Ding}{Wang
  et~al\mbox{.}}{2021}]%
        {wang2021gnnadvisor}
\bibfield{author}{\bibinfo{person}{Yuke Wang}, \bibinfo{person}{Boyuan Feng},
  \bibinfo{person}{Gushu Li}, \bibinfo{person}{Shuangchen Li},
  \bibinfo{person}{Lei Deng}, \bibinfo{person}{Yuan Xie}, {and}
  \bibinfo{person}{Yufei Ding}.} \bibinfo{year}{2021}\natexlab{}.
\newblock \showarticletitle{GNNAdvisor: An Efficient Runtime System for GNN
  Acceleration on GPUs}. In \bibinfo{booktitle}{\emph{OSDI}}.
\newblock


\bibitem[\protect\citeauthoryear{Xiong, Xiong, Chen, Jiang, and Zheng}{Xiong
  et~al\mbox{.}}{2021}]%
        {xiong2021graph}
\bibfield{author}{\bibinfo{person}{Jiacheng Xiong}, \bibinfo{person}{Zhaoping
  Xiong}, \bibinfo{person}{Kaixian Chen}, \bibinfo{person}{Hualiang Jiang},
  {and} \bibinfo{person}{Mingyue Zheng}.} \bibinfo{year}{2021}\natexlab{}.
\newblock \showarticletitle{Graph neural networks for automated de novo drug
  design}.
\newblock \bibinfo{journal}{\emph{Drug Discovery Today}} \bibinfo{volume}{26},
  \bibinfo{number}{6} (\bibinfo{year}{2021}), \bibinfo{pages}{1382--1393}.
\newblock


\bibitem[\protect\citeauthoryear{Yang, Long, Smola, Sadagopan, Zheng, and
  Zha}{Yang et~al\mbox{.}}{2011}]%
        {social}
\bibfield{author}{\bibinfo{person}{Shuang-Hong Yang}, \bibinfo{person}{Bo
  Long}, \bibinfo{person}{Alex Smola}, \bibinfo{person}{Narayanan Sadagopan},
  \bibinfo{person}{Zhaohui Zheng}, {and} \bibinfo{person}{Hongyuan Zha}.}
  \bibinfo{year}{2011}\natexlab{}.
\newblock \showarticletitle{Like like alike: joint friendship and interest
  propagation in social networks}. In \bibinfo{booktitle}{\emph{WebConf}}.
\newblock


\bibitem[\protect\citeauthoryear{Ying, Bourgeois, You, Zitnik, and
  Leskovec}{Ying et~al\mbox{.}}{2019}]%
        {ying2019gnnexplainer}
\bibfield{author}{\bibinfo{person}{Rex Ying}, \bibinfo{person}{Dylan
  Bourgeois}, \bibinfo{person}{Jiaxuan You}, \bibinfo{person}{Marinka Zitnik},
  {and} \bibinfo{person}{Jure Leskovec}.} \bibinfo{year}{2019}\natexlab{}.
\newblock \showarticletitle{Gnnexplainer: Generating explanations for graph
  neural networks}. In \bibinfo{booktitle}{\emph{NeurIPS}}.
\newblock


\bibitem[\protect\citeauthoryear{Yuan, Tang, Hu, and Ji}{Yuan
  et~al\mbox{.}}{2020}]%
        {xgnn_kdd20}
\bibfield{author}{\bibinfo{person}{Hao Yuan}, \bibinfo{person}{Jiliang Tang},
  \bibinfo{person}{Xia Hu}, {and} \bibinfo{person}{Shuiwang Ji}.}
  \bibinfo{year}{2020}\natexlab{}.
\newblock \showarticletitle{XGNN: Towards Model-Level Explanations of Graph
  Neural Networks}. In \bibinfo{booktitle}{\emph{SIGKDD}}.
\newblock


\bibitem[\protect\citeauthoryear{Zhao, Liu, Wang, Yu, and Jiang}{Zhao
  et~al\mbox{.}}{2022}]%
        {zhao2022learning}
\bibfield{author}{\bibinfo{person}{Tong Zhao}, \bibinfo{person}{Gang Liu},
  \bibinfo{person}{Daheng Wang}, \bibinfo{person}{Wenhao Yu}, {and}
  \bibinfo{person}{Meng Jiang}.} \bibinfo{year}{2022}\natexlab{}.
\newblock \showarticletitle{Learning from Counterfactual Links for Link
  Prediction}. In \bibinfo{booktitle}{\emph{ICML}}.
\newblock


\end{thebibliography}
\end{document}